%% file: bhaskara19.tex
\documentclass[final,12pt]{colt2019} 

\title{Approximate Guarantees for Dictionary Learning}
\usepackage{times}

\usepackage{fullpage, graphicx}
\usepackage{algorithm, algorithmic}
\usepackage{nicefrac}
\usepackage{enumerate}

\include{marginnotes}

\include{macros}

\coltauthor{%
 \Name{Aditya Bhaskara} \Email{bhaskara@cs.utah.edu}\\
 \Name{{Wai Ming} Tai}\thanks{Thanks to support by NSF 1514520} \Email{wmtai@cs.utah.edu}\\
 \addr University of Utah, Salt Lake City, Utah, USA
}

\begin{document}

\maketitle

\begin{abstract}%
\input{abstract}
\end{abstract}


\section{Introduction}\label{sec:intro}
\input{intro}

\section{Preliminaries and overview}\label{sec:prelim}
\input{prelim}

\section{Algorithm and analysis}\label{sec:algo}
\input{alg}

\section{Approximability of the threshold correlation problem}\label{sec:tc}
\input{hardness}

\section{Dictionary learning with outliers}\label{sec:outlier-main}
\input{outlier-main}










\bibliography{bhaskara}

\appendix
\input{appendix}

\section{Dictionary learning with outliers}\label{sec:outlier}
\input{outlier}

\end{document}

%% file: marginnotes.tex
\newif\ifnotes\notesfalse

\ifnotes
\usepackage{color}
\definecolor{mygrey}{gray}{0.50}
\newcommand{\notename}[2]{{\textcolor{mygrey}{\footnotesize{\bf (#1:} {#2}{\bf ) }}}}

\else

\newcommand{\notename}[2]{{}}

\fi

\newcommand{\anote}[1]{{\notename{Aditya}{\color{blue}#1}}}

%% file: macros.tex
\newtheorem*{theorem*}{Theorem}

\newtheorem*{proposition*}{Proposition}
\newtheorem*{lemma*}{Lemma}
\newtheorem*{conjecture*}{Conjecture}

\newtheorem*{fact*}{Fact}

\newtheorem*{hypothesis*}{Hypothesis}

\newtheorem*{claim*}{Claim}

\newtheorem*{remark*}{Remark}
\newtheorem{observation}[theorem]{Observation}
\newtheorem*{observation*}{Observation}

\newcommand{\eat}[1]{}

\newcommand{\R}{\mathbb{R}}

\newcommand{\su}[1]{^{(#1)}}

\newcommand{\poly}{\mathrm{poly}}

\newcommand{\iprod}[1]{\langle#1\rangle}

\newcommand{\Esymb}{\mathbb{E}}

\DeclareMathOperator*{\E}{\Esymb}

\newcommand{\eps}{\varepsilon}
\renewcommand{\epsilon}{\varepsilon}

\newcommand{\opt}{\textsc{Opt}}

\newcommand{\setdef}[2]{\left\{ #1 \mid #2 \right\}}
\newcommand{\inner}[2]{\left\langle #1,#2 \right\rangle}

\newcommand{\norm}[1]{\left\lVert#1\right\rVert}

%% file: abstract.tex
In the dictionary learning (or sparse coding) problem, we are given a collection of signals (vectors in $\R^d$), and the goal is to find a ``basis'' in which the signals have a sparse (approximate) representation. The problem has received a lot of attention in signal processing, learning, and theoretical computer science. The problem is formalized as factorizing a matrix $X (d \times n)$ (whose columns are the signals) as $X = AY$, where $A$ has a prescribed number $m$ of columns (typically $m \ll n$), and $Y$ has columns that are $k$-sparse (typically $k \ll d$).  
Most of the known theoretical results involve assuming that the columns of the unknown $A$ have certain {\em incoherence} properties, and that the coefficient matrix $Y$ has random (or partly random) structure. 

The goal of our work is to understand what can be said in the absence of such assumptions. Can we still find $A$ and $Y$ such that $X \approx AY$? We show that this is possible, if we allow violating the bounds on $m$ and $k$ by appropriate factors that depend on $k$ and the desired approximation. Our results rely on an algorithm for what we call the {\em threshold correlation} problem, which turns out to be related to hypercontractive norms of matrices. We also show that our algorithmic ideas apply to a setting in which some of the columns of $X$ are outliers, thus giving similar guarantees even in this challenging setting.

%% file: intro.tex
The problem of dictionary learning or {\em sparse coding} was introduced in the neuroscience literature by~\citep{Olshausen1997Sparse} as a way of understanding how the visual cortex perceives images. Their hypothesis was that it learns to recognize ``basic patterns'' that allow for a sparse representation of all natural images.  This idea was subsequently used extensively in the machine learning and signal processing literature, with applications including compression, feature extraction, denoising and more.  Mathematically, the sparse coding problem is the following: given a collection of vectors $x_1, x_2, \dots, x_n \in \R^d$, does there exist a small `dictionary' (small set of vectors in $\R^d$) such that every $x_i$  can be written as a {\em sparse combination} of the elements of the dictionary? Formally, if $X \in \R^{d\times n}$ is the matrix whose columns are the vectors $x_i$ and $m, k$ are parameters, then the goal is to find $A \in \R^{d \times m}$ such that $X \approx AY$, for some matrix $Y$, each of whose columns is $k$-sparse. 

The early algorithms for the problem, including ones by~\citep{Olshausen1997Sparse} and~\citep{ksvd} are iterative in nature, and find a decomposition that minimizes $\norm{X - AY}_F$, subject to sparsity constraints. These algorithms lacked theoretical convergence guarantees. More recently, there has been significant progress in terms of designing algorithms that are efficient (polynomial time and sample complexity), and are guaranteed to recover an unknown dictionary $A$. The work of Spielman~\citep{Spielman2013Exact} studied the setting in which $A$ is full column rank (which implies $m \le d$), and $Y$ has entries that are independently drawn from a sparse random distribution, and $k < \sqrt{d}$. Many later works including~\citep{Arora2014New, Agarwal2014Learning, Barak2015Dictionary, Gribonval},  have developed powerful algorithmic tools that handle the case $d > m$, as well as more general distributions of $Y$. Recently,~\citep{AwasthiV18} considered semi-random models in which a small portion of $Y$ is random and the rest can be adversarial (in support). All of these results have the following high level structure: assuming that $A$ satisfies certain properties (full-rank, incoherent columns, etc.) and that $Y$ as certain random (or partly random) structure, one can efficiently recover $A$ up to a desired accuracy, in polynomial time. The recoverability of the dictionary $A$ is known as {\em identifiability}, and is an important requirement in all the algorithms discussed above.

The question we consider is the following: suppose we only wish to find $A$ and $Y$ such that $\norm{X - AY}_F$ is small. Can we solve the problem even in the cases where identifiability fails to hold? Note that while identifiability is important for some applications where the dictionary elements are believed to have semantic information, it is not crucial for applications such as compression, denoising, etc. Here, the recoverability of $X$ up to a small error is more important. Our main result is showing that such a decomposition can be efficiently obtained, with small losses in the parameters $m$ and $k$ that depend on the quality of the approximation desired. We show that such a dependence is necessary in order to avoid known intractability results~\citep{Tillmann}.
Our approximation guarantees are similar in spirit to bi-criteria approximations known for problems such as $k$-means clustering (see~\citep{Ostrovsky2006Effectiveness, Jaiswal2012simple}). We will now describe in detail the problem settings we consider along with more motivation.

\subsection{Problem formulations and motivation}
\paragraph{Approximate dictionary learning}
Suppose we are given a collection of signals $x_1, x_2, \dots, x_n \in \R^d$, with the promise that there exists a dictionary $A^*$ of $m$ vectors, such that each of the $x_i$ can be approximated by a $k$-sparse linear combination of the dictionary elements, so that the total error in the approximation is at most $\gamma^* \norm{X}_F^2$. (Here $\gamma^*$ can be sub-constant.)  The goal is to find a dictionary $A'$ with $m'$ vectors, such that each $x_i$ can be approximated by a $k'$-sparse linear combination of the dictionary elements, and the total error in the approximation is $\le (\gamma^* + \eps) \norm{X}_F^2$, for some prescribed $\eps >0$. We highlight that there are {\em no additional assumptions} on the vectors or the combination (except a norm bound $\Lambda$, as we will see).

This way of formalizing the question is similar to the one done in the work of~\citep{Olshausen1997Sparse} which introduced sparse coding.  As long as $m'$ and $k'$ are close enough to $m, k$, the obtained representation $A' Y'$  would still achieve a non-trivial ``compression'' of the original matrix.  Intuitively, we are promised that $X$ has a ``representation complexity'' (measured by the number of parameters used to approximate it well) of $d m + k n$ (which is $\ll nd$, the na\"ive representation), and we are aiming to construct a good approximation with roughly $d m' + k' n$ parameters. 

Our results (Theorem~\ref{thm:main-alg}) show how to obtain $m'$ and $k'$ that are small, when $k \ll d$. 
One way of looking at our result is that if the number of observed vectors ($n$) is large, then assuming there exists a basis in which each vector can be "compressed" to $k$ real numbers, we can efficiently find a basis in which each vector can be expressed with $k'$ ($\approx k^2/\poly(\eps)$ (Corollary~\ref{col:main})) real numbers, up to an $\eps$ error. In the setting where $k$ is small ($\ll \sqrt{d}$), this can be a significantly smaller representation.
One key thing to consider is the dependence on $\eps$. The bounds we give on $m', k'$ have $\poly(1/\eps)$ factors in them. It is natural to ask if a sub-polynomial, say $\log (1/\eps)$ dependence is possible. However, we note that existing hardness results (see~\citep{Tillmann}) rule out this possibility. Indeed, even for the ``easier'' problem of sparse recovery,~\citep{FosterKT15} showed that such a guarantee is impossible.

Our algorithm proceeds by finding $A'$ one column at a time. To this end, the following problem turns out to be of importance. We note that such a question implicitly arises in recent works such as~\citep{AwasthiV18}. 

\newcommand{\ball}{\mathcal{B}_2^d}

\paragraph{The threshold correlation problem}
We now define the problem that plays a central role in our algorithm for dictionary learning. In what follows, $\ball$ denotes the unit ball in $\R^d$, i.e., $\ball = \{ x\in \R^d : \norm{x} \le 1\}$.   

Suppose we are given $n$ unit vectors $v_1, v_2, \dots, v_n \in \ball$. We can define the vector that is  {\em most correlated} with these vectors as the $x \in \R^d$ with $\norm{x}=1$ that maximizes $\sum_i \iprod{x, v_i}^2$. Such a vector can clearly be found as the top left singular vector of the matrix whose columns are $v_i$. Now, consider a variant in which we are also given a threshold $\tau$, and we wish to maximize $\sum_i \iprod{x, v_i}^2$, but {\em only over indices} $i$ that satisfy $\iprod{x, v_i}^2 \ge \tau$. Intuitively, we only ``get credit'' if the squared inner product with $v_i$ is above a threshold. We call this the threshold correlation ($\tau$-TC) problem, and formally define it below (with weights).

\begin{definition}[$\tau$-TC problem]
Let $v_1, v_2, \dots, v_n \in \ball$, and let $w_i \in \R_{\ge 0}$ be non-negative weights. Let $\tau \in [0, 1]$ be given. The goal in the {\em threshold correlation} ($\tau$-TC) problem is to find a unit vector $x \in \R^d$ that maximizes
\[ \sum_{i \in [n]} \left[ \mathbf{1}_{\iprod{x, v_i}^2 \ge \tau} \right] w_i \iprod{x, v_i}^2,\]
where $[\mathbf{1}_P]$ for a predicate $P$ is the indicator that is $1$ if $P$ holds and $0$ otherwise.
\end{definition}

The two extremes $\tau = 0$ and $\tau = 1$ can be solved easily. The former is simply the problem of computing the largest singular vector. The latter is solved by returning the most frequent vector (taking into account the weights $w_i$) of length exactly $1$ among the $v_i$. Developing algorithms for arbitrary $\tau \in [0,1]$ is a natural question. In our setting, it turns out that bi-criteria approximations for the $\tau$-TC problem give interesting guarantees for dictionary learning. 

\begin{definition}[$(\alpha, \beta)$-approximation]\label{def:bicriteria}
Let $v_1, v_2, \dots, v_n$, $w_1, \dots, w_n$ and $\tau$ be the inputs for an instance of the $\tau$-TC problem. Suppose the optimum objective value is $\opt$. For $0 \le \alpha, \beta \le 1$, a $(\alpha, \beta)$-approximation algorithm is one that returns a unit vector $x \in \R^d$ such that 
\[ \sum_{i \in [n]} \left[ \mathbf{1}_{\iprod{x, v_i}^2 \ge \alpha \cdot \tau} \right] w_i \iprod{x, v_i}^2  \ge \beta \cdot \opt.   \]
In other words, the threshold is reduced by a factor $\alpha$, and the objective value by a factor $\beta$.
\end{definition}

While natural by itself, the $\tau$-TC problem turns out to be related to the well-studied question of approximating hypercontractive norms of matrices (see e.g., \citep{BarakBHKSZ12, Bhaskara2011Approximating, Bhattiprolu2018Inapproximability}), as we will see.

\paragraph{Approximate dictionary learning with outliers}
Finally, we consider the dictionary learning problem where some of the columns (upper bounded in number) are allowed to be {\em outliers}, i.e., they do not need to have a sparse approximation in terms of the dictionary. This is quite common to assume in practice. 

\paragraph{The objective.} The first question is to define the right objective for capturing this version of the problem. The natural choice is to use the same objective as before, i.e., find $A$ and $Y$ so as to minimize $\norm{X - AY}_F^2$ (subject to the constraints as before). However, consider a scenario in which all the columns of $X$ are unit vectors, and say 10\% of the columns are outliers. Suppose the optimum objective value is $0.05 \norm{X}_F^2$, with {\em all the error} coming from the outliers. (I.e., the inliers have a perfect representation in terms of a sparse dictionary.)  Now, if an algorithm achieves the optimum objective value of $0.05 \norm{X}_F^2$, it could potentially be because of {\em every column} having an error of $0.05$ in the obtained representation. This is not desirable, as the dictionary is not doing well enough on the inliers.

To remedy this, we propose the following objective, which is closely related to the objectives studied in various robust estimation problems.

\begin{definition}[Dictionary learning with outliers]\label{def:dict-outlier}
Given parameters $m, k$, and a collection of observations $x_1, x_2, \dots, x_n \in \R^d$ that form the columns of the matrix $X$, the goal is to decompose $X$ as $A Y + N$, where $A \in \R^{d \times m}$, every column of $Y$ has at most $k$ non-zero entries, and $N$ is a matrix with at most $\rho n$ nonzero columns. The objective is to minimize the error in the decomposition, i.e., minimize $\norm{X - AY - N}_F^2$.
\end{definition}

While we technically think of the fraction $\rho$ as small, our algorithm makes no such assumption. This is not too surprising, given that the guarantee we aim for is additive in the norm of the entire matrix. The problem of allowing outliers in dictionary learning was studied in~\citep{Gribonval}, but the focus there is quite different from ours. 

\subsection{Our results}
Our main result gives a connection between an approximation for the $\tau$-TC problem and dictionary learning.

\begin{theorem}\label{thm:main-alg}
Let $X \in \R^{d \times n}$ be an instance of dictionary learning for which we know that there exist matrices $A^* \in \R^{d\times m}, Y^* \in \R^{m \times n}$, where \begin{enumerate}[(a)]
\setlength{\itemsep}{0pt}
\item the columns of $A^*$ are all unit vectors,
\item each column of $Y^*$ is $k$-sparse, and satisfies the norm bound $\norm{Y_i^*}^2 \le \Lambda \norm{x_i}^2$, and 
\item we have $\norm{X - A^* Y^*}_F^2 \le \gamma^* \norm{X}_F^2$.  
\end{enumerate}
Let $\eps > 0$ be an accuracy parameter, and suppose that there exists an efficient $(\alpha, \beta)$-approximation algorithm for the $\tau$-TC problem, when $\tau = \eps^2 / k\Lambda$ (for parameters $\alpha, \beta$ depending on $\tau$). Then there is an efficient algorithm that outputs matrices $A', Y'$ such that (a) $A'$ has at most $O(m  \Lambda/ \beta \eps)$ columns, (b) every column of $Y'$ has at most $O(k \Lambda/\alpha \eps^2)$ non-zero entries, and (c) $\norm{X - A'Y'}_F^2 \le (\gamma^* + \eps) \norm{X}_F^2$.
\end{theorem}

Using our approximation algorithm for the $\tau$-TC problem, we obtain the following corollary.

\begin{corollary}~\label{col:main}
Under the assumptions of Theorem~\ref{thm:main-alg}, the number of columns of the output $A'$ is at most $O\left( \frac{m k^2 \Lambda^3}{\eps^5} \right)$. Also, every column of the output $Y'$ has at most $O\left( \frac{k^2 \Lambda^2}{\eps^4} \right)$.
\end{corollary}

Note that these bounds are off from $m, k$ by factors that only depend on $k$, $\eps$ and $\Lambda$ (and not $d, m$). For instance, if $k$ is small and $\eps, \Lambda$ are constants, we obtain a non-trivial {\em compression} of the input signals. As noted before, a $\text{poly}(1/\eps)$ dependence on the parameter $\eps$ is essential due to hardness results. 

\paragraph{Aside: the parameter $\Lambda$.} Our bounds also depend on the parameter $\Lambda$, which is the norm of the coefficient vector used to represent $x_i / \norm{x_i}$ in the promised solution. For intuition on $\Lambda$, suppose we have a unit vector $v$ (in our case, an $x_i$) that we express as a linear combination of unit vectors $u_1, u_2, \dots, u_k$ (in our case a subset of the columns of $A^*$), $v = \xi_1 u_1 + \dots + \xi_k u_k$.  If $u_i$ are (near) orthogonal, we can {\em guarantee} that $\norm{\xi} = O(1)$. However, near orthogonality (or large least-singular-value) is not a necessity, depending on $v$. The factor $\Lambda$ arises in our analysis as it is related to how well $v$ is correlated with the $u_i$.  As an extreme example, consider $d=2$, and $v = (0,1)$, $u_1 = (1, 0)$ and $u_2 = (\sqrt{1-\gamma^2}, \gamma)$, where $\gamma$ is tiny. In this case, $v$ can be represented as a linear combination of $u_1$ and $u_2$, but neither of the vectors is ``well-correlated'' with $v$. This is, in essence, because the coefficients used in the linear combination are large ($\Lambda = \Theta(1/\gamma^2)$). Having small $\Lambda$ allows us to argue that one of the $u_i$ has high correlation with $v$. (See Lemma~\ref{lem:increment-vector}.) Thus a dependence on $\Lambda$ is essential for any iterative algorithm of the type we consider.

The next result we show is a bi-criteria approximation for the $\tau$-TC problem, in the sense of Definition~\ref{def:bicriteria}.

\begin{theorem}\label{thm:algo-tau-tc}
For every $\tau \in (0,1)$, there is a polynomial time $\left( \frac{\tau}{4}, \frac{\tau^2}{32} \right)$-approximation algorithm for the $\tau$-TC problem.
\end{theorem}

It is a very interesting open problem to improve at least one of the parameters in the bi-criteria approximation. We can show that we cannot expect constant factor approximations. To this end, we prove the following result, connecting the complexity  of $\tau$-TC to the question of approximating hypercontractive norms of matrices.  Informally, we show that constant factor bi-criteria approximations to the $\tau$-TC problem imply a logarithmic approximation to the so-called $2 \mapsto p$ norms (which is unlikely, due to the work of~\cite{BarakBHKSZ12}).  We refer to Section~\ref{sec:hardness} for the details. 

Finally, for the problem of dictionary learning with outliers, we show that an algorithm almost identical to the one in Theorem~\ref{thm:main-alg} (with a slightly modified analysis) yields the following.

\begin{theorem}\label{thm:outlier}
Let $X \in \R^{d \times n}$ be a matrix of observations for which we know that there exist matrices $A^* \in \R^{d\times m}, Y^* \in \R^{m \times n}, N^* \in \R^{d \times n}$, where 
\begin{enumerate}[(a)]
\setlength{\itemsep}{0pt}
\item the columns of $A^*$ are all unit vectors,
\item each column of $Y$ is $k$-sparse, and satisfies the norm bound $\norm{Y_i}^2 \le \Lambda \norm{x_i}^2$,
\item the matrix $N^*$ has at most $\rho n$ non-zero columns, and
\item we have $\norm{X - A^* Y^*}_F^2 \le \gamma^* \norm{X}_F^2$.  
\end{enumerate}
Let $\eps > 0$ be an accuracy parameter, and suppose that there exists an efficient $(\alpha, \beta)$-approximation algorithm for the $\tau$-TC problem, when $\tau = \eps^2 / k\Lambda$ (for parameters $\alpha, \beta$  depending on $\tau$). Then there is an efficient algorithm that outputs matrices $A', Y', N'$ such that (a) $A'$ has at most $O(m  \Lambda/ \beta \eps^3)$ columns, (b) every column of $Y'$ has at most $O(k \Lambda/\alpha \eps^2)$ non-zero entries, (c) $N'$ has at most $\rho n$ non-zero columns, and (d) $\norm{X - A'Y'}_F^2 \le (\gamma^* + \eps) \norm{X}_F^2$.
\end{theorem}

Note that the only difference in the parameters between Theorems~\ref{thm:main-alg} and~\ref{thm:outlier} is a mildly worse dependence on $\eps$ in the bound on the number of columns. 

\subsection{Related work}
There has been a large body of work in the signal processing literature on the problem of dictionary learning and also the related question of sparse recovery.  Sparse recovery is the problem of reconstructing a sparse vector $x$ given $Ax$, where $A$ is now a {\em known} dictionary. Iterative {\em pursuit}  algorithms of the kind we consider have been extensively studied in this context (see, e.g.,~\citep{Davis1997} and~\citep{Tropp}). Our iterative algorithms are also reminiscent of Frank-Wolfe iteration, which is a powerful technique for sparse approximation. We refer to~\citep{Clarkson2010Coresets, Shalev} and references therein. In this context, our results may be viewed as showing that greedy pursuit with an appropriate subroutine ---an algorithm for the $\tau$-TC problem--- can give approximate guarantees for dictionary learning.

Another related work is that of~\citep{Blum2016Sparse}. They consider a problem in which we are given $x_1, x_2, \dots, x_n \in \R^d$, and the goal is to come up with a small set of points $\mathcal{P}$ such that all the $x_i$ lie in conv$(\mathcal{P})$, the convex hull of $\mathcal{P}$.  This is equivalent to requiring a decomposition in which the columns of the $Y$ matrix are convex combinations. However, it turns out that the results do not imply any bound for dictionary learning (in the traditional setting that we study). 

Our work shows an interesting connection between the $\tau$-TC problem and hypercontractive norms. Indeed, our algorithm implies an approximation algorithm for the $2 \mapsto p$ norm problem, in the case when the norm is {\em large enough}. Now, there are other more sophisticated sub-exponential time algorithms for hypercontractive norms~\cite{BarakBHKSZ12}. It is an interesting question to investigate if these techniques imply better approximations for the $\tau$-TC problem. From our results, this would translate to better guarantees for approximate dictionary learning.

%% file: prelim.tex
\subsection{Notation}
We will refer to the input vectors $x_i$ to the dictionary learning problem as observations, signals, or simply as input vectors. We will also use standard notation for norms of vectors and matrices. $\norm{v}_p$ refers to the $\ell_p$ norm. We will drop the subscript when we are referring to the $\ell_2$ norm. For a matrix $A$, we use $A_i$ to refer to the $i$'th column vector of $A$. $\norm{A}_F$ is the Frobenius norm of the matrix, defined as $\sqrt{\sum_{i,j} A_{ij}^2}$. 

\subsection{Paper outline}
The rest of the paper is organized as follows. In Section~\ref{sec:algo}, we give our main algorithm for approximate dictionary learning, assuming an approximation algorithm for the $\tau$-TC problem. This will establish Theorem~\ref{thm:main-alg}. The outline of the argument is presented at the opening of the section.

Next in Section~\ref{sec:tc}, we give a bi-criteria approximation for the threshold correlation problem, establishing~\ref{thm:algo-tau-tc}. The argument involves a simple clustering procedure, and shows that one of the given vectors is itself a good enough solution for the $\tau$-TC problem.  We then establish a connection between the $\tau$-TC problem and the question of approximating $2\mapsto p$ norms of matrices, thereby showing a hardness result for $\tau$-TC.

Finally in Section~\ref{sec:outlier} (whose brief outline is in Section~\ref{sec:outlier-main}), we consider the problem of dictionary learning with outliers defined above, and show that the iterative framework of Section~\ref{sec:algo} extends to a setting in which some of the columns have no sparse approximation using the chosen dictionary. The argument relies heavily on the methods in Section~\ref{sec:algo}, even though the structure is slightly different.


%% file: alg.tex
In this section, we will describe the algorithm in detail, along with its analysis. The goal will be to establish Theorem~\ref{thm:main-alg}. Let us start with a rough outline of the algorithm, and introduce some notation. Recall that $X \in \R^{d \times n}$ is the matrix whose columns are the observations $x_i \in \R^{d}$. The goal is to return an approximation decomposition $A' Y'$, as described in Theorem~\ref{thm:main-alg}.

\paragraph{Algorithm outline.}  The algorithm proceeds by building the matrix $A'$ iteratively, one column at a time, and simultaneously the matrix $Y'$ one row at a time. In each iteration, we maintain a ``current approximation'' to each of the $x_i$ (which will be the $i$th column of the current $A' Y'$). This will be the projection of $x_i$ onto a {\em subset} of the columns of $A'$.  We denote the error in the approximation in iteration $t$ by $z_i\su{t}$.  I.e., if $A'$ and $Y'$ are the matrices at time $t$, then $z_i\su{t} = x_i - (A'Y')_i$.  We then find a vector $v$ that has a high correlation with many of the $z_i\su{t}$ by solving an appropriate instance of the $\tau$-TC problem. The $v$ is then added as a new column to $A'$. A new row is then added to $Y'$, and the $i$th entry is non-zero only if $|\iprod{v, z_i\su{t}}|$ is above an appropriately chosen threshold.  This procedure is formally described in Algorithm~\ref{alg:main}.

\paragraph{Analysis.}  There are two main steps in the analysis. The first is showing that the {\em residual} problem always has a good solution for $v$.  It turns out that this is a non-trivial step in the analysis because a natural set-cover style argument, which involves considering the optimal dictionary $A^*$ and showing that one of the columns allows us to make sufficient progress, faces two difficulties.  Firstly, the columns of $A^*$ are arbitrary. There turns out to be a significant difference in the guarantee one can obtain when the columns of $A^*$ are (near) orthogonal, and when they are arbitrary. Another (more serious) difficulty arises due to the fact that the residual vector $z_i\su{t}$ is a projection of $x_i$ onto some (potentially arbitrary) subspace $T$ (formed by a subset of the columns in the current $A'$). While it is not difficult to argue that one of the columns of $A^*$ has a high inner product with $x_i$, it is trickier to say the same about $z_i\su{t}$.

The second step in the analysis is to use the guarantees of the bi-criteria approximation to the $\tau$-TC problem to bound the number of iterations and the sparsity. Both of these bounds are relatively straightforward. The relation between the number of iterations and the error is shown via an analysis similar to that of Frank-Wolfe iteration (see~\cite{Clarkson2010Coresets}). As for the sparsity, the manner in which we update $Y'$ (non-zero only if $\iprod{z_i\su{t}, v}$ is above a threshold) implies that every column in $Y'$ is sparse.


\begin{algorithm} 
\caption{$\textsc{DictApprox} (A \in \R^{d \times n}, k, m, \Lambda, \eps)$}
\begin{algorithmic}[1]\label{alg:main}
\STATE Initialize $z_i\su{0} = x_i$. Set $\tau = \frac{\eps^2}{k\Lambda}$, and $M = m k ..$
\STATE Initialize $A', Y'$ to empty matrices
\FOR{$t = 1:M$}
\STATE Let $v$ be the output of an $(\alpha , \beta)$-approximation algorithm for $\tau$-TC, with parameters being: 
\[ \tau = \frac{\eps^2}{k\Lambda} \text{, the vectors }~ \frac{z_i\su{t}}{\norm{x_i}}, \text{ and weights } w_i = \norm{x_i}^2, \text{ for } 1 \le i \le n.\]
\STATE Add $v$ as a column to $A'$
\STATE Add an empty row to $Y'$
\FOR{$i \in [n]$ satisfying $\iprod{z_i\su{t}, v}^2 \ge \alpha \tau$}
	\STATE\label{alg-step-z} set $z_i\su{t+1} = z_i\su{t} - \iprod{z_i\su{t}, v} v$
	\STATE set $Y'[t, i] = \iprod{z_i\su{t}, v}$
\ENDFOR
\ENDFOR
\STATE Return $A', Y'$
\end{algorithmic}
\end{algorithm}

Let us now proceed with the details of the analysis. Apart from the notation $X, A^*, Y^*, \eps$ from the statement of Theorem~\ref{thm:main-alg}, we also define the following quantities: 
\begin{itemize}
\item $z_i\su{t}$ is the ``residual'' vector, as described in the algorithm (and the outline above).
\item $\gamma_i := \norm{(X - A^*Y^*)_i}^2 / \norm{x_i}^2$, i.e., the optimum error in column $i$, as a fraction of its mass. Clearly, \[ \gamma^* = \frac{1}{\norm{X}_F^2} \sum_i \norm{x_i}^2 \gamma_i.\]
If we view $\norm{x_i}^2/\norm{X}_F^2$ as defining a probability distribution over indices, the above is equivalent to $\gamma^* = \E_i [ \gamma_i].$
\item $S_i^*$ is the support of $Y_i^*$, i.e., the subset of the dictionary used to represent $x_i$ in the optimum solution.
\item $\theta_i\su{t} = \norm{z_i\su{t}}^2/ \norm{x_i}^2$, i.e., the error in column $i$ at time $t$, as a fraction of its mass.
\item Let $\psi\su{t}$ (which will be the potential) be $\E_i [ \theta_i\su{t} ]$, i.e.,
\[ \psi\su{t} := \frac{1}{\norm{X}_F^2} \sum_i \norm{x_i}^2 \theta_i \su{t} = \frac{1}{\norm{X}_F^2} \sum_i \norm{z_i\su{t}}^2.\]
\item For any real number $\xi$, we denote $(\xi)_+ = \max\{ 0, \xi \}$.
\end{itemize}
The quantity $\psi\su{t}$ defined above will be the potential function we use in the analysis. Let us first start with a simple observation:

\begin{observation}
At any step, $z_i\su{t}$ is the projection of $x_i$ onto the subspace orthogonal to the space spanned by a subset of the columns of the current $A'$.
\end{observation}
\begin{proof}
The proof follows immediately from the fact that $z_i\su{0} = x_i$, and the update step (step~\ref{alg-step-z}). 
\end{proof}

The observation also makes sure that the instance of $\tau$-TC we are solving is indeed valid. I.e., the vectors $z_i\su{t}/\norm{x_i}$ are in the unit ball.

Our goal now is to show the following bound on the convergence of the algorithm:
\begin{theorem}\label{thm:algo-final-bound}
After $t$ iterations of the main loop in the algorithm, the error $\psi\su{t}$ is bounded by $\gamma^* + \frac{16 m \Lambda}{\beta t}$. Further, the number of non-zero entries in each column of $Y'$ at the end of the algorithm is at most $\nicefrac{1}{\alpha \tau}$, where $\tau = \eps^2/k\Lambda$ as in the algorithm.
\end{theorem}

This clearly implies Theorem~\ref{thm:main-alg} by setting $t$ appropriately. The goal is thus to show Theorem~\ref{thm:algo-final-bound}. 

\subsection{Increment lemma}
As described in the outline, a key step of our argument is showing that in each step, there exists a vector with a large inner product with a large number of the residual vectors.

\begin{lemma}\label{lem:increment-vector}
Let $u, s_1, s_2, \dots, s_k$ be unit vectors in $\R^d$, and let $u = \sum_i \alpha_i s_i + z$.  Let $T$ be any subspace of $\R^d$, and define $\theta = \norm{u - \Pi_T u}$.  Then, we have 
\[ \sum_i \iprod{u - \Pi_T u, s_i}^2 \ge \frac{ (\theta^2 - \norm{z}^2)_+^2}{4 (\sum_i \alpha_i^2)},\]
where $(\beta)_+$ denotes $\max(0, \beta)$, for any real number $\beta$.
\end{lemma}

\begin{proof}
We may assume that $\theta^2 \ge \norm{z}^2$, else there is nothing to prove. Now define $x = \sum_i \alpha_i s_i$ and $u^{\perp} := u - \Pi_T u$, for convenience. We will show that 
\begin{equation}\label{eq:tmp111}
\iprod{u^{\perp}, x} \ge \frac{\theta^2 - \norm{z}^2}{2},
\end{equation} 
and the desired result then follows by an application of the Cauchy-Schwartz inequality, which implies that $\left( \sum_i \alpha_i \iprod{u^\perp, s_i} \right) \le (\sum_i \alpha_i^2) (\sum_i \iprod{u^\perp, s_i}^2 )$. We can now show~\eqref{eq:tmp111} as follows.
\begin{align*}
\iprod{u^{\perp}, x} &= \iprod{u, x} - \iprod{\Pi_T u, x} \\
&= \frac{\norm{u}^2 + \norm{x}^2 - \norm{z}^2}{2} - \iprod{\Pi_T u, x} \\
&\ge \frac{\norm{u}^2 + \norm{x}^2 - \norm{z}^2}{2} - \norm{\Pi_T u} \norm{x} \\
&\ge \frac{ 1 + \norm{x}^2 - \norm{z}^2}{2} - \frac{\norm{\Pi_T u}^2 + \norm{x}^2}{2} \\
& = \frac{ \norm{u^{\perp}}^2 - \norm{z}^2}{2}.
\end{align*}
In the last step, we used the Pythagoras theorem, concluding that $1 - \norm{\Pi_T u}^2 = \norm{u^{\perp}}^2$. This concludes the proof of the lemma.
\end{proof}

As an immediate consequence we have the following (recall that $S_i^*$ is the subset of the dictionary vectors used to represent $x_i$ in the optimum solution).

\begin{lemma}\label{lem:change-i}
Let $S_i^* = \{ s_{i,1}^*, s_{i,2}^*, \dots, s_{i,k}^* \}$.  For every $i \in [n]$ and iteration $t$, we have the following:
\[ \sum_{j \in [k]} \iprod{z_i\su{t}, s_{i,j}^*}^2 \ge \frac{ (\theta_i\su{t} - \gamma_i)_+^2}{4\Lambda} \norm{x_i}^2. \]
\end{lemma}
This follows directly from Lemma~\ref{lem:increment-vector}. The next lemma is a convexity argument, relating the sum of the ``per-column'' improvements to the overall progress. We defer the proof to Appendix~\ref{app:col-to-matrix}.

\begin{lemma}\label{lem:col-to-matrix}
In every iteration $t$, we have
\[ \sum_{i \in [n]} (\theta_i\su{t} - \gamma_i)_+^2 \norm{x_i}^2 \ge (\psi\su{t} - \gamma^*)_+^2 \norm{X}_F^2.\]
\end{lemma}

The next lemma is crucial to the argument. It shows the existence of a good solution to the instance of the $\tau$-TC problem, unless the current error is already close to the optimum.

\begin{lemma}\label{lem:optimal-vector}
Consider any iteration $t$ for which $\psi\su{t} \ge \gamma^* + \eps$. There exists a vector $a\in A^*$, and a subset $R$ of $[n]$ such that
\[ \iprod{ z_i\su{t}, a}^2 \geq \frac{\eps^2}{16 k\Lambda} \norm{x_i}^2  \text{ for all $i\in R$,} \]
and additionally,
\[ \sum_{i\in R} \iprod{z_i\su{t}, a}^2 \geq \frac{(\psi\su{t} - \gamma^*)^2}{16m \Lambda } \norm{X}_F^2.\]
\end{lemma}

\newcommand{\diff}{(\psi\su{t} - \gamma^*)^2}

\begin{proof}
The rough idea of the proof is to use Lemma~\ref{lem:change-i} along with Lemma~\ref{lem:col-to-matrix} to conclude that there exists a column $a$ of the optimal dictionary $A^*$ such that $\sum_{i} \iprod{z_i\su{t}, a}^2$ is large. However, we also need every term to be large enough. This requires additional pruning steps. We defer the details to Appendix~\ref{app:optimal-vector}.
\end{proof}

As a corollary, we show that the vector found by the algorithm satisfies appropriate guarantees.

\begin{corollary}\label{lem:alg-guarantee}
Let $v$ be the vector output by an $(\alpha, \beta)$-approximation to the $\tau$-TC problem, with
\[ \tau = \frac{\eps^2}{16 k \Lambda}, \text{ the vectors } \frac{z_i\su{t}}{\norm{x_i}}, \text{ and weights } w_i = \norm{x_i}^2, \text{ for } 1 \le i \le n.\]
Then there exists a subset $R' \subseteq [n]$ such that $\iprod{z_i\su{t}, v}^2 \ge \alpha \tau \norm{x_i}^2$ for all $i \in R'$, and 
\[ \sum_{i \in R'} \iprod{z_i\su{t}, v}^2 \ge \frac{\beta \diff}{16 m \Lambda} \norm{X}_F^2. \]
\end{corollary}

The corollary follows from Lemma~\ref{lem:optimal-vector} and the definition of an $(\alpha, \beta)$-approximation (Definition~\ref{def:bicriteria}).
Next, we are ready to prove Theorem~\ref{thm:algo-final-bound}. 
 
\begin{proof}[of Theorem~\ref{thm:algo-final-bound}]
%
%
From Corollary~\ref{lem:alg-guarantee}, we have that
\[ \psi\su{t} - \psi\su{t+1} \ge C \diff \text{, where } C = \frac{\beta}{16 m \Lambda}. \]
Writing $a_t = (\psi\su{t} -\gamma^*)$, the above may be written as $a_t - a_{t+1} \ge C \cdot a_t^2 \ge C \cdot a_t a_{t+1}$.  Rearranging, 
\[ \frac{1}{a_{t+1}}  - \frac{1}{a_t} \ge C. \]
Using $a_0 \le 1$, we get that for all $t$, $\frac{1}{a_{t}} \ge Ct$, or
\[ \psi\su{t} -\gamma^* = a_t \le \frac{1}{Ct} = \frac{16 m \Lambda}{\beta t}.\]

This completes the proof of the error bound. 
%
Next, to bound the sparsity per column, note that every time we use one of the vectors $v$ to update $z_i\su{t}$, the quantity $\norm{z_i\su{t}}^2$ drops by at least $\alpha \tau \norm{x_i}^2$.  Thus the number of such updates must be at most $1/(\alpha \tau)$. 
This completes the proof of the theorem.
\end{proof}

%% file: hardness.tex
We first present our main algorithmic result for the $\tau$-TC problem (Theorem~\ref{thm:algo-tau-tc}). It is a bi-criteria approximation, in the sense of Definition~\ref{def:bicriteria}.

\subsection{Bi-criteria approximation}
The key lemma behind our algorithm is the following.

\begin{lemma}\label{lem:clustering}
Let $v_1, v_2, \dots, v_q \in \ball$, and suppose that there exists a unit vector $x$ such that $\iprod{ x, v_i}^2 \ge \tau$ for all $i \in [q]$.  Suppose also that we are given weights $w_i \ge 0$ for each $i$. Then there exists an index $\ell$ such that
\[  \sum_{i} w_i \iprod{v_{\ell}, v_i}^2 [ \mathbf{1}_{\iprod{v_\ell, v_i}^2 \ge \tau^2/4}] \ge \frac{\tau^2}{32} \sum_i w_i \iprod{x, v_i}^2. \]
\end{lemma}

First, let us see how Theorem~\ref{thm:algo-tau-tc} follows from Lemma~\ref{lem:clustering}.

\begin{proof}[of Theorem~\ref{thm:algo-tau-tc}]
Consider any instance of the $\tau$-TC problem, and let $x$ be the optimum vector. Let $S$ be the set of indices $\{ i \in [n] : \iprod{x, v_i}^2 \ge \tau \}$.  Applying Lemma~\ref{lem:clustering} to this set of vectors gives us a $v_{\ell}$ satisfying the conclusion of the lemma. As this is a vector of length $\le 1$, scaling it to a unit vector only improves the objective. Thus the algorithm can simply iterate over all the vectors $v_i$ and pick the one that satisfies the conclusion of the lemma. This completes the proof.
\end{proof}

Let us thus prove Lemma~\ref{lem:clustering}.

\begin{proof}[of Lemma~\ref{lem:clustering}]
For every $i \in [q]$, let us write $v_i = \alpha_i x + u_i$, with $\iprod{x, u_i}=0$. By assumption, $\alpha_i^2 \ge \tau$ for all $i \in [q]$.  Let $I \subset [q]$ be the set of indices with $\alpha_i > 0$. By replacing $x$ by $-x$ if necessary, we may assume that
\begin{equation}\label{eq:tmp5}
\sum_{i \in I} w_i \alpha_i^2 \ge \frac{1}{2} \sum_{i \in [q]} w_i \alpha_i^2.
\end{equation}

For any two $i, j \in I$, we have 
\begin{equation}
\iprod{v_i, v_j} = \alpha_i \alpha_j + \iprod{u_i, u_j}. \label{eq:iprod-relation}
\end{equation}
Thus we have that $\iprod{v_i, v_j} < \nicefrac{\alpha_i \alpha_j}{2} \implies \iprod{u_i, u_j} < -\nicefrac{\alpha_i \alpha_j}{2} \le -\nicefrac{\tau}{2}$.  Now, consider the following clustering procedure for the vectors $v_i$, $i \in I$. Start with any unclustered $v_i$, and place all $j$ such that $\iprod{v_i, v_j} \ge \nicefrac{\alpha_i \alpha_j}{2}$ in cluster $C_i$ (in this case, we say that $i$ ``captures'' $j$). We repeat this procedure until no unclustered $v_i$ remains. We claim that in this process, at most $2\lceil \frac{1}{\tau} \rceil + 1$ clusters can be formed. To see this, consider the indices $i$ that occur in the procedure above, and call them $i_1, i_2, \dots, i_t$. Because $i_{r+1}$ is not captured by $i_1, \dots, i_r$, we have $\iprod{v_{i_a}, v_{i_b}} < (\alpha_{i_a} \alpha_{i_b})/2$ for all $a, b \in [t]$. From~\eqref{eq:iprod-relation}, this means that $\iprod{ u_{i_a}, u_{i_b}} < -\nicefrac{\tau}{2}$ for all $a, b \in [t]$.

Now, we appeal to the simple observation that for any integer $k \ge 1$, we cannot have $k+1$ unit vectors whose pairwise inner products are all $< -\frac{1}{k}$. (This contradicts the fact that the Gram matrix is positive semidefinite.)

Thus, at most $(2\lceil \frac{1}{\tau} \rceil+1) < \nicefrac{4}{\tau}$ clusters are formed. This means that there exists one cluster $C$, say the one formed using the vector $v_\ell$, such that
\begin{equation}\label{eq:tmp3}
\sum_{i \in C} w_i \alpha_i^2 \ge \frac{\tau}{4} \cdot \sum_{i \in I} w_i \alpha_i^2.
\end{equation}
(I.e., if we view $w_i \alpha_i^2$ as the {\em mass} of index $i$, one of the clusters has at least $\tau/4$ of the mass, due to the bound on the number of clusters.)

Next, for all $i \in C$, we have (by the way clustering was performed) that 
\begin{equation}\label{eq:tmp2}
\iprod{v_\ell, v_i}^2 \ge \frac{ \alpha_\ell^2 \alpha_i^2}{4} \ge \frac{\tau \alpha_i^2}{4}.
\end{equation}
The RHS is always $\ge \tau^2/4$. This implies that 
\begin{equation}\label{eq:tmp4}
\sum_{i \in [q]} w_i \iprod{v_{\ell}, v_i}^2 [ \mathbf{1}_{\iprod{v_\ell, v_i}^2 \ge \tau^2/4}] 
\ge \sum_{i \in C} w_i \iprod{v_{\ell}, v_i}^2
\end{equation}

Now using~\eqref{eq:tmp2} followed by~\eqref{eq:tmp3} and finally~\eqref{eq:tmp5}, we obtain
\[ \sum_{i \in [q]} w_i \iprod{v_{\ell}, v_i}^2 [ \mathbf{1}_{\iprod{v_\ell, v_i}^2 \ge \tau^2/4}] ~\ge~ \frac{\tau}{4} \sum_{i \in C} w_i \alpha_i^2 ~\ge~ \frac{\tau^2}{16} \sum_{i \in I} w_i \alpha_i^2 ~\ge~ \frac{\tau^2}{32} \sum_{i \in [q]} w_i \alpha_i^2. \]
This completes the proof.
\end{proof}

\subsection{Inapproximability -- connection to matrix norms}\label{sec:hardness}
The main result of this section is to prove that the approximability (even for bi-criteria approximations) of the $\tau$-TC problem is related to the well-studied question of approximating matrix norms.

Recall the problem of computing $q \mapsto p$ norm of a matrix. 

\begin{definition}\label{def:qtop}
Let $p, q \ge 1$, and let $A \in \R^{n \times d}$ be a matrix. We define the $q \mapsto p$ norm of $A$ as
\[ \norm{A}_{q \mapsto p} = \max_{x \in \R^d,~ x \ne 0} \frac{\norm{Ax}_p}{\norm{x}_q}. \]
\end{definition}

We will specifically consider the case of $q = 2$ and $p > 2$. This is a special case of the so-called {\em hypercontractive} norms, which refers to the case $p > q$. Note that a $\theta > 0$ approximation to the $q \mapsto p$ norm problem is a vector $x$ such that
\[ \frac{\norm{Ax}_p}{\norm{x}_q} \ge \theta \norm{A}_{q \mapsto p}. \]

\begin{theorem}\label{thm:hardness}
Suppose we have an $(\alpha, \beta)$-approximation to the $\tau$-TC problem, where $\alpha, \beta$ are independent of $\tau$ (e.g., constants). Then there exists an $\Omega\left(\frac{\alpha^{1/p-1/2} \beta^{1/p}}{\log^{1/p} n}\right)$ approximation to the $2 \mapsto p$ norm problem.
\end{theorem}
We defer the proof to Appendix~\ref{app:thm:hardness}.


\paragraph{Known hardness results for the $2\mapsto 4$ norm. }  \citep{BarakBHKSZ12} proved that assuming the exponential time hypothesis (ETH), it is impossible to obtain a $2^{(\log n)^{1/4}}$ approximation to the $2\mapsto 4$ norm problem for $n \times n$ matrices in time $< n^{(\log n)^{1/2}}$ (which is super-polynomial). Together with Theorem~\ref{thm:hardness}, this implies that achieving constant $\alpha, \beta$ via polynomial time algorithms is impossible assuming ETH.

%% file: outlier-main.tex
It turns out that the technique of iteratively {\em peeling off} mass is useful in obtaining approximations even when we have outliers among the signals $x_i$. We will indeed be able to use ideas very similar to those in Section~\ref{sec:algo} to obtain efficient algorithms  even when a $\rho$ fraction of the columns of $X$ are arbitrary outliers.

The full details will be presented in Appendix~\ref{sec:outlier}. We now mention a few key differences. First, we can longer just use the fact that if a lot of mass is uncovered, then we must be able to make progress (because even in the optimal solution, a lot of the {\em overall} mass is uncovered). In essence, we need to be able to perform the earlier algorithm and the analysis only on the inliers, without knowing the inliers! This turns out to be possible, by maintaining proxies for the appropriate parameters that arise in the analysis. 

We defer the details to Appendix~\ref{sec:outlier}.

%% file: appendix.tex
\section{Proof of Lemma~\ref{lem:col-to-matrix}}\label{app:col-to-matrix}
\begin{proof}
We may assume that $\psi\su{t} \ge \gamma^*$, as there is nothing to prove otherwise.

Now, using the fact that $\E[Z^2] \ge (\E[Z])^2$ for any random variable $Z$, we have that
\begin{equation}\label{eq:convexity}
\sum_i \frac{\norm{x_i}^2}{\norm{X}_F^2} (\theta_i\su{t} - \gamma_i)_+^2 \ge \left(  \sum_i \frac{\norm{x_i}^2}{\norm{X}_F^2} (\theta_i\su{t} - \gamma_i)_+ \right)^2.
\end{equation}
Next, using the fact that $(\beta)_+ \ge \beta$, we have
\[ \sum_i \frac{\norm{x_i}^2}{\norm{X}_F^2} (\theta_i\su{t} - \gamma_i)_+ \ge \sum_i \frac{\norm{x_i}^2}{\norm{X}_F^2} (\theta_i\su{t} - \gamma_i) = \psi\su{t} - \gamma^*.\]
We have noted that $\psi\su{t} \ge \gamma^*$, and thus the RHS is $\ge 0$.  Thus we can square the inequality above, and together with~\eqref{eq:convexity}, we have
\[ \sum_i \frac{\norm{x_i}^2}{\norm{X}_F^2} (\theta_i\su{t} - \gamma_i)_+^2 \ge (\psi\su{t} - \gamma^*)^2. \]
This completes the proof.
\end{proof}

\section{Proof of Lemma~\ref{lem:optimal-vector}}\label{app:optimal-vector}
The rough idea of the proof is to use Lemma~\ref{lem:change-i} along with Lemma~\ref{lem:col-to-matrix} to conclude that there exists a column $a$ of the optimal dictionary $A^*$ such that $\sum_{i} \iprod{z_i\su{t}, a}^2$ is large. However, we also need every term to be large enough. This requires additional pruning steps.

First, we perform a pruning on the indices $i \in [n]$.  Let 
\[  I := \left\{ i \in [n] : (\theta_i\su{t} - \gamma_i)_+^2 \ge \frac{\diff}{2} \right\}. \]

By definition, the following holds (the LHS is a sum over indices {\em not in} $I$):
\[ \sum_{i \in [n] \setminus I} (\theta_i\su{t} - \gamma_i)_+^2 \norm{x_i}^2 < \frac{\diff}{2} \norm{X}_F^2. \]
Thus from Lemma~\ref{lem:col-to-matrix}, we have that 
\begin{equation}\label{eq:tmp6}
\sum_{i \in I} (\theta_i\su{t} - \gamma_i)_+^2 \norm{x_i}^2 \ge \frac{1}{2} \sum_{i \in [n]} (\theta_i\su{t} - \gamma_i)_+^2 \norm{x_i}^2.
\end{equation}

Next, for each $i \in I$, we consider the bound from Lemma~\ref{lem:change-i}, and prune out terms that contribute too little to the sum. For every $i\in I$, define $J_i$ as
\[  J_i = \left\{ j \in [k] : \iprod{z_i\su{t}, s_{i,j}^*}^2 \ge \frac{(\theta_i\su{t} - \gamma_i)_+^2}{8k\Lambda} \norm{x_i}^2 \right\}. \]
The same argument as above (because $|S_i^*| = k$) implies that
\begin{equation}\label{eq:tmp7}
\sum_{j \in J_i} \iprod{z_i\su{t}, s_{i,j}^*}^2 \ge \frac{1}{2} \sum_{j \in [k]} \iprod{z_i\su{t}, s_{i,j}^*}^2.
\end{equation}
Now, if we imagine each $i \in I$ contributing a mass of $\iprod{z_i\su{t}, s_{i,j}^*}^2$ to the bin $s_{i,j}^*$ for each $j \in J_i$, we can conclude that there exists a bin (one of the $[m]$ columns in the dictionary) with a total mass that is at least
\[ \frac{1}{m} \sum_{i \in I} \sum_{j \in J_i} \iprod{z_i\su{t}, s_{i,j}^*}^2. \]
Combining equations~\eqref{eq:tmp6} and~\eqref{eq:tmp7} (and Lemmas~\ref{lem:change-i} and \ref{lem:col-to-matrix}), this is at least 
\[ \frac{1}{16m \Lambda} \sum_{i \in [n]} (\theta_i\su{t} - \gamma_i)_+^2 \norm{x_i}^2 \ge \frac{\diff}{16 m\Lambda} \norm{X}_F^2.\]

Furthermore, for each $i \in I$ and $j\in J_i$, by our definitions, we have 
\[ \iprod{z_i\su{t}, s_{i,j}^*}^2 \ge \frac{\diff}{16k\Lambda} \norm{x_i}^2 \ge \frac{\eps^2}{16k \Lambda} \norm{x_i}^2.\]
This completes the proof.

\section{Proof of Theorem~\ref{thm:hardness}}\label{app:thm:hardness}

\begin{proof}
Consider an instance $A$ of the $2 \mapsto p$ norm problem. 
Let $v_i$ denote the rows of $A$. 
Without loss of generality, we assume that $\norm{v_i} \leq 1$ by dividing $A$ by $\max_i \norm{v_i}$.
We wish to find $x$ with $\norm{x} = 1$ such that $\sum_{i \in [n]} \iprod{v_i, x}^p$ is maximized.


Let $Q$ denote the optimum value of $(\sum_i \iprod{v_i, x}^p)^{1/p}$.
Namely, $\sum_i \iprod{v_i, x_0}^p = Q^p$ for some unit vector $x_0$.
Denote $I = \setdef{i\in [n]}{\inner{v_i}{x_0}^p \geq \frac{Q^p}{2n}}$.
It is easy to see that $\sum_{i\in I} \inner{v_i}{x_0}^p \geq Q^p/2$.

By dividing the terms $\iprod{v_i, x}^2$ into levels (by powers of $2$), we have that there exists a level $(\tau_j, \tau_{j+1})$ such that the sum 
\[ \sum_{i \in I, ~ \tau_j \le \iprod{v_i, x_0}^2 \le \tau_{j+1}} \iprod{v_i, x_0}^p \ge \frac{Q^p}{2(\log n+1)} \]
where $\tau_j = \frac{Q^p}{2n} \cdot 2^j$ for $j=0,1,\dots,(\log n)+1$.

Now, suppose we consider the instance of the $\tau$-TC problem with this value of $\tau=\tau_j$, and the vectors $v_i$ (with its corresponding weight $1$).  Let $\opt_{TC}$ denote the optimum value of this problem. We clearly have 
\[ \opt_{TC} = \max_{\norm{x}=1}  \sum_{i \in [n],~ \iprod{v_i, x}^2 \ge \tau} \iprod{v_i, x}^2 \ge \frac{Q^p\tau}{2(\log n+1)(2\tau)^{p/2}}. \]

Thus, consider an $(\alpha, \beta)$-approximation to the $\tau$-TC problem. This will end up finding a vector $x$ such that 
\[  \sum_{i \in [n],~ \iprod{v_i, x}^2 \ge \alpha \tau} \iprod{v_i, x}^2 \ge \beta \cdot \frac{Q^p\tau}{2(\log n+1)(2\tau)^{p/2}}. \]

For any such vector, we have 
\[ \sum_{i \in [n]}  \iprod{v_i, x}^p \ge \sum_{i \in [n], ~\iprod{v_i, x}^2 \ge \alpha \tau}  (\alpha \tau)^{(p-2)/2} \iprod{v_i, x}^2 \ge (\alpha \tau)^{(p-2)/2} \cdot \frac{\beta Q^p\tau}{2(\log n+1)(2\tau)^{p/2}} = \frac{\alpha^{(p-2)/2} \beta Q^p}{2^{p/2+1}(\log n+1)}. \] 
\end{proof}

%% file: outlier.tex
As we discussed earlier, the technique of iteratively {\em peeling off} mass is useful in obtaining approximations even when we have outliers among the signals $x_i$.  We will now prove this formally, thus establishing Theorem~\ref{thm:outlier}. 
We begin by recalling the definitions of $A^*, Y^*$, and $N^*$ (the latter is the matrix that only consists of the outlier columns, and is a $d \times n$ matrix that is non-zero in only $\rho n$ columns). Additionally, we define the following:
\newcommand{\II}{\mathcal{I}}
\newcommand{\psitil}{\widehat{\psi}}
\begin{enumerate}
\item $\II$: the indices of the inliers (of course, unknown to the algorithm).
\item $X_{\II}$: the submatrix of $X$ restricted to the set of inliers. It has dimensions $d \times (1-\rho)n$.
\item As before, $\gamma^* = \frac{1}{\norm{X_{\II}}_F^2} \sum_{i \in \II} \norm{(X - A^* Y^*)_i}^2 $.
\item As before, $\psi\su{t} = \frac{1}{\norm{X_{\II}}_F^2} \sum_{i \in \II} \norm{z_i\su{t}}^2$, where $z_i\su{t}$ is the residual (as maintained by the algorithm) in column $i$ at iteration $t$.
\item $\Phi\su{t} = \sum_{i \in [n]} \norm{z_i\su{t}}^2$.
\item $\psitil\su{t} = \min_{T \subset [n],~ |T| = (1-\rho)n} \sum_{i \in T} \norm{z_i\su{t}}^2$.
\end{enumerate}

The two new parameters are $\Phi\su{t}$ and $\psitil\su{t}$. These are quantities the algorithm can compute at any time step. $\psitil\su{t}$ is a quantity that we use as a proxy for $\norm{X_{\II}}_F^2 \cdot \psi\su{t}$. For any $t$, we have
\[ \psitil\su{t} \le \norm{X_{\II}}_F^2 \cdot \psi\su{t}. \]

\paragraph{Algorithm.}  The algorithm is almost identical to the one from Section~\ref{sec:algo}, with one main difference: instead of running for a prescribed number of iterations, we execute the main loop as long as 
\[ \frac{1}{\norm{X}_F^2}(\Phi\su{t} - \Phi\su{t+1} ) \ge \frac{\beta \eps^3}{16 m \Lambda}. \]
As another minor difference, in the end, the algorithm declares the $\rho n$ columns with the largest values of $\norm{z_i\su{t}}$ as the outliers (which is the natural choice). The procedure is described in detail in Algorithm~\ref{alg:main:outlier}. 

\begin{algorithm}
\caption{$\textsc{OutlierDictApprox} (A \in \R^{d \times n}, k, m, \Lambda, \rho, \eps)$}
\begin{algorithmic}[1]\label{alg:main:outlier}
\STATE Initialize $z_i\su{0} = x_i$. Set $\tau = \frac{\eps^2}{k\Lambda}$.
\STATE Initialize $A', Y'$ to empty matrices
\WHILE{$\Phi\su{t+1} < \Phi\su{t} - \frac{\beta \eps^3}{16 m \Lambda}$}
\STATE Let $v$ be the output of an $(\alpha, \beta)$-approximation algorithm for $\tau$-TC, with parameters being: 
\[ \tau = \frac{\eps^2}{k\Lambda} \text{, the vectors }~ \frac{z_i\su{t}}{\norm{x_i}}, \text{ and weights } w_i = \norm{x_i}^2, \text{ for } 1 \le i \le n.\]
\STATE Add $v$ as a column to $A'$
\STATE Add an empty row to $Y'$
\FOR{$i \in [n]$ satisfying $\iprod{z_i\su{t}, v}^2 \ge \alpha \tau$}
	\STATE set $z_i\su{t+1} = z_i\su{t} - \iprod{z_i\su{t}, v} v$
	\STATE set $Y'[t, i] = \iprod{z_i\su{t}, v}$
\ENDFOR
\ENDWHILE
\STATE Return $A', Y'$, and declare the $\rho n$ columns with the largest $\norm{z_i\su{t}}$ as the outliers.
\end{algorithmic}
\end{algorithm}

\begin{observation}\label{obs:psitil}
At any iteration $t$, the current decomposition (using the current $A', Y'$ and declaring the columns with largest $z_i\su{t}$ norm as outliers) achieves error $\psitil\su{t}$.
\end{observation}

The proof is immediate from the definition of $\psitil\su{t}$. One consequence of this is that we may assume that $\psitil\su{0} \ge \eps \norm{X}_F^2$ (i.e., the Frobenius norm of the entire matrix $X$).
Otherwise, we return $A'$, $X'$ being $0$ matrices (with say one column and row respectively) and declare the $\rho n$ vectors of largest norm as outliers $N'$ which clearly satisfies the guarantee in Theorem~\ref{thm:outlier}. 

We can now directly use our arguments from Section~\ref{sec:algo}.

\begin{proof}[of Theorem~\ref{thm:outlier}]
Consider any iteration $t$.  Clearly, we have that $\Phi\su{t} - \Phi\su{t+1} \ge \norm{X_\II}^2 (\psi\su{t} - \psi\su{t+1})$. This is simply because for the columns not in $\II$, $\norm{z_i\su{t+1}} \le \norm{z_i\su{t}}$. Thus the (known) quantity $\Phi\su{t}$ is used as a proxy for the (unknown) quantity $\psi\su{t}$.

Now, from Corollary~\ref{lem:alg-guarantee}, we have that if $\psi\su{t} \geq \gamma^* + \eps$,
\begin{equation}
    \Phi\su{t} - \Phi\su{t+1} \ge  
    \norm{X_{\II}}_F^2  (\psi\su{t} - \psi\su{t+1}) \ge 
    \frac{\beta \norm{X_{\II}}_F^2}{16 m \Lambda} \cdot (\psi\su{t} - \gamma^*)^2 \ge 
    \frac{\beta \eps^2 \norm{X_\II}_F^2}{16 m \Lambda} \ge \frac{\beta \eps^3 \norm{X}_F^2}{16 m \Lambda}.
\end{equation}
In the very last inequality we used the observation from above that $\psitil\su{0} \ge \eps \norm{X}_F^2$, and clearly, $\norm{X_\II}_F^2 \ge \psitil\su{0}$, because the definition of $\psitil$ ignores the $\rho n$ columns of largest length (while $\norm{X_\II}$ ignores the outliers).

This means that when the algorithm terminates, we have $\psi\su{t} < \gamma^* + \eps$, and therefore also \[ \psitil\su{t} < \norm{X_{\II}}_F^2 (\gamma^* + \eps) = \norm{X - A^* Y^* - N^*}_F^2 + \eps \norm{X_\II}_F^2. \]

This establishes that when the algorithm terminates, we have a good bound on the error in our approximation.  Next, we need to bound the number of iterations.  For this, we simply use the crude bound obtained by the drop in $\Phi\su{t}$.  $\Phi\su{0} \le \norm{X}_F^2$, and $\Phi\su{t} \ge 0$ for any $t$.  Thus the number of iterations is at most
$\frac{16 m \Lambda}{\eps^3 \beta}$. 

The proof of the sparsity per column of  the output $Y'$ is precisely the same as in Section~\ref{sec:algo}, as the update to $z$ as well as the value of $\tau$ are the same.
\end{proof}


